\documentclass[10pt,twocolumn,letterpaper]{article}

\usepackage{cvpr}              

\usepackage{pifont}
\definecolor{tabfirst}{rgb}{1, 0.7, 0.7}
\definecolor{tabsecond}{rgb}{1, 0.85, 0.7}
\definecolor{tabthird}{rgb}{1, 1, 0.7}
\usepackage[accsupp]{axessibility}
\usepackage{makecell}
\usepackage{animate}
\usepackage{graphicx}	
\usepackage{amsmath}	
\usepackage{amssymb}	
\usepackage{amsthm}
\usepackage{booktabs}
\usepackage{cuted}
\usepackage{times}
\usepackage{epsfig}
\usepackage{caption}
\usepackage{float}
\usepackage{placeins}
\usepackage{color, colortbl}
\usepackage{stfloats}
\usepackage{enumitem}
\usepackage{tabularx}
\usepackage{xstring}
\usepackage{multirow}
\usepackage{xspace}
\usepackage{url}
\usepackage{subcaption}
\usepackage[hang,flushmargin]{footmisc}
\usepackage{kotex}
\usepackage{arydshln}
\usepackage[export]{adjustbox}
\usepackage{wrapfig}
\usepackage{algorithm,algorithmicx,algpseudocode}
\usepackage{listings}
\usepackage{bm}
\usepackage{mdframed}
\usepackage{array}
\usepackage{tcolorbox}
\usepackage{pgf} 

\newlength\paramargin
\newlength\abovetabcapmargin
\newlength\belowtabcapmargin
\newlength\abovefigcapmargin
\newlength\belowfigcapmargin
\newlength\aboveeqmargin
\newlength\beloweqmargin

\usepackage{amsmath}

\def\w{{\mathbf{w}}}

\def\z{{\mathbf{z}}}

\def\Sc{{\mathcal{S}}}

\newcommand{\softmax}{\mathrm{softmax}}

\newcommand{\code}[1] {\texttt{#1}}

\usepackage{thmtools,thm-restate}
\declaretheorem[]{proposition}
\declaretheorem[]{lemma}

\def\Oc{{\mathcal{O}}}
\def\Rd{{\mathbb{R}}}
\def\Ed{{\mathbb{E}}}

\definecolor{GoogleBlue}{RGB}{66, 133, 244}
\definecolor{GoogleRed}{RGB}{234, 67, 53}
\definecolor{GoogleYellow}{RGB}{251, 188, 4}
\definecolor{GoogleGreen}{RGB}{52, 168, 83}
\newcommand{\method}{VPS}

\newcommand{\vpst}[1]{\cellcolor{GoogleBlue!10}#1}
\newcommand{\vpsf}[1]{\cellcolor{GoogleBlue!20}#1}

\definecolor{cornellred}{rgb}{0.7, 0.11, 0.11}
\definecolor{cadmiumgreen}{rgb}{0.0, 0.42, 0.24}
\definecolor{aliceblue}{rgb}{0.91, 0.94, 0.97}
\definecolor{darkblue}{rgb}{0.83, 0.89, 0.97}
\definecolor{Red7}{rgb}{0.941, 0.243, 0.243}
\definecolor{Green7}{RGB}{55, 178, 77}
\definecolor{Blue9}{rgb}{0.098,0.3,0.9}

\definecolor{codegreen}{rgb}{0,0.6,0}
\definecolor{codegray}{rgb}{0.5,0.5,0.5}
\definecolor{codepurple}{rgb}{0.58,0,0.82}
\definecolor{backcolour}{rgb}{0.95,0.95,0.92} 

\lstdefinestyle{pytorchstyle}{
    language=Python,
    backgroundcolor=\color{backcolour},
    commentstyle=\color{codegreen}\itshape,
    keywordstyle=\color{codepurple}\bfseries,
    numberstyle=\tiny\color{codegray},
    stringstyle=\color{codepurple},
    basicstyle=\ttfamily\footnotesize, 
    breakatwhitespace=false,         
    breaklines=true,                 
    captionpos=b,                    
    keepspaces=true,                 
    numbers=left,                    
    numbersep=5pt,                  
    showspaces=false,                
    showstringspaces=false,
    showtabs=false,                  
    tabsize=2,
    frame=lines, 
    framesep=2mm,
    morekeywords={torch, stack, sum, expand, cat, multinomial, softmax} 
}

\definecolor{cvprblue}{rgb}{0.21,0.49,0.74}
\usepackage[pagebackref,breaklinks,colorlinks,allcolors=cvprblue]{hyperref}

\def\paperID{****} 
\def\confName{CVPR}
\def\confYear{2026}

\title{
Video Parallel Scaling: Aggregating Diverse Frame Subsets for VideoLLMs
}
\author{Hyungjin Chung$^1${\quad} 
Hyelin Nam$^{1,2}${\quad} 
Jiyeon Kim$^3${\quad} 
Hyojun Go$^{1,4}${\quad} 
Byeongjun Park$^1$\\  
Junho Kim$^1${\quad} 
Joonseok Lee$^1${\quad} 
Seongsu Ha$^{1,5}${\quad} 
Byung-Hoon Kim$^{1,6}$
\\
$^1$EverEx\,\,\,
$^2$University of Michigan\,\,\,
$^3$KAIST\,\,\,
$^4$ETH Z{\"u}rich\,\,\,\\
$^5$UNC Chapel Hill\,\,\,
$^6$Yonsei University\\
}

\begin{document}
\maketitle

\begin{strip}
    \centering
    \resizebox{\textwidth}{!}{
    \includegraphics[width=\linewidth]{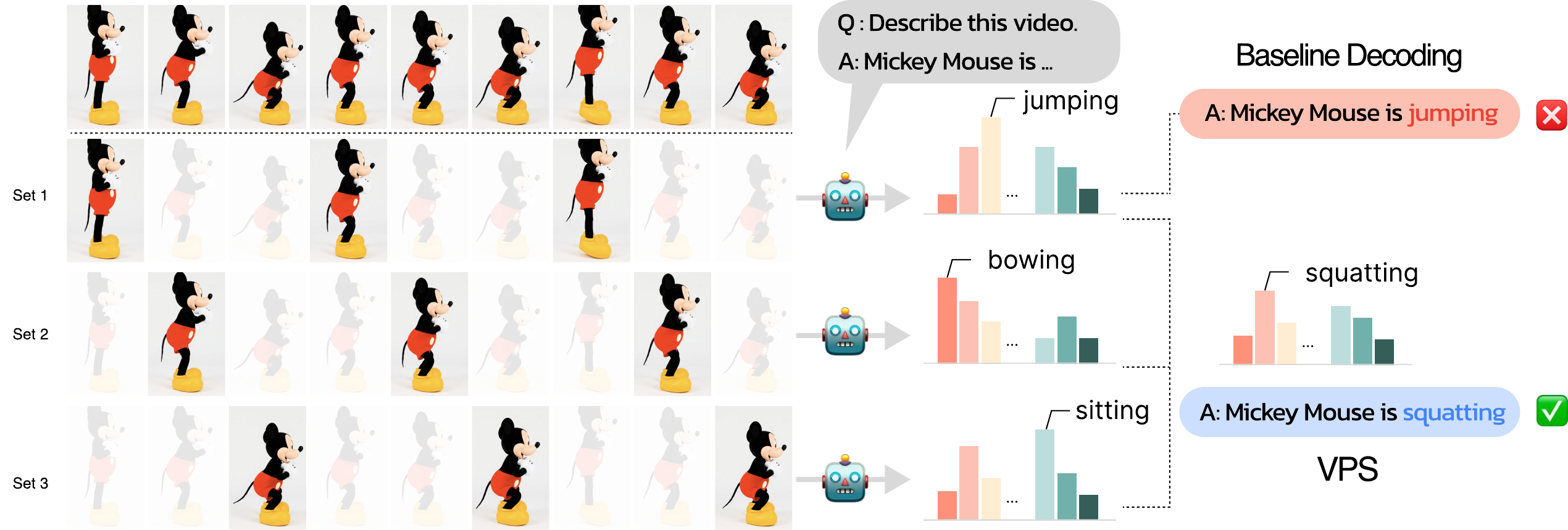}
    }
    \vspace{\abovefigcapmargin}
    \captionof{figure}{
    \textbf{Conceptual Illustration of \method.} VideoLLMs take as input subsampled frames from the original video (Video generated with Veo3~\citep{DeepMind2025Veo3}.), limiting their understanding capabilities. Increasing the sampled frames within context leads to computation/memory issues or decrease in performance. In contrast, \method~keeps the number of subsampled frames, and scales the number of streams {\em in parallel},
    with each stream attending to different frames. By aggregating the probability values from each stream, \method\ results in 1) fine-grained motion perception, 2) consistent gains with more streams, without increasing the context length.}
    \label{fig:concept}
\end{strip}

\begin{abstract}
Video Large Language Models (VideoLLMs) face a critical bottleneck: increasing the number of input frames to capture fine-grained temporal detail leads to prohibitive computational costs and performance degradation from long context lengths. We introduce Video Parallel Scaling (VPS), an inference-time method that expands a model's perceptual bandwidth without increasing its context window. VPS operates by running multiple parallel inference streams, each processing a unique, disjoint subset of the video's frames. By aggregating the output probabilities from these complementary streams, VPS integrates a richer set of visual information than is possible with a single pass. We theoretically show that this approach effectively contracts the Chinchilla scaling law by leveraging uncorrelated visual evidence, thereby improving performance without additional training. Extensive experiments across various model architectures and scales (2B-32B) on benchmarks such as Video-MME and EventHallusion demonstrate that VPS consistently and significantly improves performance. VPS scales more favorably than other parallel alternatives (e.g. Self-consistency) and is complementary to other decoding strategies, offering a memory-efficient and robust framework for enhancing the temporal reasoning capabilities of VideoLLMs.
\end{abstract}

\section{Introduction}
\label{sec:intro}

Video Large Language Models (VideoLLMs) inherit the cross-modal reasoning abilities of image-based VLMs, extending the models with multi-frame inputs. This access to temporal information enables them to (i) reason about causality and motion by tracing how an object’s state evolves, (ii) localize or summarize long-horizon events that unfold over minutes and hours, and (iii) carry out temporal reasoning that demands ordering, counting, or cross-referencing disjoint moments: capabilities fundamentally out of reach for single-frame systems. The latest closed models such as Gemini-2.5~\citep{comanici2025gemini} and GPT-4o~\citep{hurst2024gpt} as well as open models such as Qwen2.5-VL~\citep{bai2025qwen2}, Gemma3~\citep{team2025gemma}, and InternVL3~\citep{zhu2025internvl3} all showcase these abilities, and recent benchmarks~\citep{nagrani2025minerva} confirm that many reasoning tasks truly require video context rather than per-frame cues.

Extending from images to video, however, {\em amplifies} several long-standing weaknesses. 
Even for Vision Language Models (VLMs) that process a single image, the image tokens constitute the vast majority of the context length~\citep{yin2025lifting,bi2025unveiling}.
Every additional frame inflates the sequence length, so that the token counts and compute budgets explode. Compute-optimal analysis shows that simply feeding more frames is rarely the most efficient use of inference FLOPs~\citep{wang2025inference}. Even when memory suffices, the accuracy almost always deteriorates with context length, mirroring the ``context-rot''~\citep{hong2025context} observed for LLMs. Fine-grained kinematics—subtle velocity changes, limb articulation, or repetition count—remain difficult~\citep{hong2025motionbench,nagrani2025minerva}. Detailed queries often elicit hallucinated or mis-ordered events~\citep{zhang2024eventhallusion}.
Moreover, unlike text-based large language models (LLMs), where deeper inference-time reasoning reliably improves accuracy~\citep{jaech2024openai,wu2024inference}, recent evidence shows that VideoLLM errors stem primarily from impoverished visual inputs~\citep{upadhyay2025timeblindness,li2025f16,liu2025more}; adding more ``thoughts'' offers limited gains unless the model's perceptual bandwidth is improved.
Hence, simply investing more {\em sequential} inference-scaling to a VideoLLM offers sharply diminishing returns; meaningful progress instead demands strategies that enlarge the model’s perceptual bandwidth.

To overcome the perceptual bottleneck without aggravating the quadratic cost of longer clips, we introduce \textbf{Video Parallel Scaling (\method)}—an inference-time strategy that trades additional parallel computation for richer visual coverage. Concretely, \method\ spawns $J$ independent streams, each sampling a different subset of frames, and fuses their predictions through a weighted aggregation. Because streams are processed concurrently, total FLOPs grow linearly with $J$ while the sequence length\footnote{This computation is also embarrassingly parallelizable.}, and hence the memory footprint remains unchanged.
By this construction, we can leverage more and more visual information by increasing the number of parallel streams, which would have been simply dropped otherwise. As illustrated in Fig.~\ref{fig:concept}, each stream provides {\em complementary} visual evidence, which, when combined together, yields a correct answer.
Furthermore, we theoretically show that \method\ offers a way to contract the scaling law~\citep{hoffmann2022training} {\em without} additional training, by showing that the loss only contracts faster when using uncorrelated subsets of frames.

Through extensive experiments covering different model sizes (2B - 32B), VideoLLM architectures (Qwen2.5-VL~\citep{bai2025qwen2}, Gemma3~\citep{team2025gemma}, InternVL3~\citep{zhu2025internvl3}), and benchmarks (Video-MME~\citep{fu2025video}, EventHallusion~\citep{zhang2024eventhallusion}), we show that \method~consistently outperforms baselines, resulting in additional improvements with more parallel streams. We further show that \method\ scales favorably to other inference-time strategies, and is often orthogonal to the other approaches such that it can be used in harmony.

\section{Background}
\label{sec:background}

\subsection{Video Large Language Models}
\label{subsec:videollms}

State-of-the-art VideoLLMs pipe a vision backbone---usually a vision transformer (ViT)~\citep{dosovitskiy2021vit} variant with positional encoding~\citep{su2021rope,press2022alibi}---into a frozen or lightly-tuned LLM head. Typical open models include Qwen-VL 2.5~\citep{bai2025qwen2}, Gemma-3~\citep{team2025gemma}, and InternVL3~\citep{zhu2025internvl3}. Each attaches a lightweight projection layer that maps per-frame patch embeddings to language tokens, enabling the downstream transformer blocks to treat visual tokens and text tokens uniformly. While open-source VideoLLMs have made rapid progress, comprehensive evaluations show they still hallucinate events, mis-count repetitive actions, and crumble when genuine temporal reasoning is required, especially on long or high-frame-rate clips~\citep{zhang2024eventhallusion,hong2025motionbench,nagrani2025minerva,fu2025video,li2025temporal}.

\paragraph{Challenges of long context}
\label{subsec:challenge_long_context}

Significant progress has been made to enable the use of long context in LLMs~\citep{xiong2024effective,chen2024longlora,team2024gemini} to a point where 128k context window is now standard in many proprietary and open checkpoints~\citep{liu2025comprehensive}. Nevertheless, feasibility does not imply proficiency. A wide collection of studies points out that the decrease in LLM performance is inevitable with longer context~\citep{li2025longcontext,hsieh2024ruler,bai2024longbench,an2025why,shi2025explaining,hong2025context}. This finding naturally translates into VideoLLMs, where increasing the number of frames for input leads to a degradation in the performance above a certain threshold~\citep{gao2025exploring,wang2025inference}, especially for VideoLLMs of modest size. Because every additional frame multiplies the token sequence length, most systems either fix a static budget (e.g. $\leq$ 32 frames) or down-sample videos to $\leq 1$ frames per second (fps). On the other hand, we would want to show the model as many video frames as possible, especially when the goal is to comprehend the intricate temporal details. 
%
%
A natural question arises: is this possible without increasing the context length? \method\ attempts to give an answer to this question with a positive, namely by constructing parallel streams with different video frames shown for each stream.

\subsection{Inference-time Enhancement Strategies}

We categorize the training-free inference-time enhancement strategies into two: 1) non-scaling approaches that have static compute and gain, 2) scaling approaches that enjoy improved performance when more compute is used.

\paragraph{Non-scaling approaches}

Several frame selection methods have been devised in order to improve upon the usual uniform sampling strategy, to manage the number of frames used for VideoLLMs within a budget. AKS~\citep{tang2025adaptive} and BOLT~\citep{liu2025bolt} use query-frame similarity as a measure to select more relevant frames. VideoTree~\citep{wang2025videotree} employs a hierarchical structure to select frames for long videos. Modifications to attention have also been devised. SlowFast-LLaVA~\citep{xu2024slowfast} uses two different sequences of video frames that are sampled with different fps and combines them with concatenation. 
The MovieChat series~\citep{song2024moviechat,song2025moviechat+} proposes to select and retain task-relevant video segments through sparse memory mechanisms.
Other methods such as modifying the attention, or leveraging external vision models for better interpretation, have also been proposed. SlowFocus~\citep{nie2024slowfocus} introduces multiple frequency mixing attention and mixed frequency sampling from relevant segment grounding. MVU~\citep{ranasinghe2025understanding} proposes to use off-the-shelf vision models to decipher the video into language, then uses the summarized information as the input context. Token merging strategies~\citep{fu2024framefusion,hyun2025multi} attempt to preserve the original accuracy while using a fraction of the tokens.

Methods that operate on the logit-probability level are also popular. Contrastive decoding (CD)~\citep{li2023contrastive,leng2024mitigating} induces a vector nudge in the logit space by contrasting positive and negative pairs. Temporal contrastive decoding (TCD)~\citep{zhang2024eventhallusion} extends CD to video sequence. RITUAL~\citep{woo2024ritual} constructs an augmented view of an image, and sums the logit values before decoding.

\paragraph{Scaling approaches}

Best-of-N (BoN) sampling~\citep{stiennon2020learning} selects the answer with the highest reward after running $N$ parallel decoding streams. Speculative rejection~\citep{sun2024fast} improves naive BoN by incorporating rejection sampling with reward models during the sampling process. Jinnai \etal \cite{jinnai2024regularized} propose Regularized BoN, and Ichihara \etal \cite{ichihara2025evaluation} extend this method into a stochastic version, showing theoretical guarantees. While effective, one can only use these variants when having access to a reward function.
Self-consistency~\citep{wang2023selfconsistency} selects the mode of the parallel streams, and hence is free from the dependence on external reward models. On the other hand, Chain-of-Thought (CoT)~\citep{kojima2022large} is a sequential scaling approach that forces the model to reason before answering directly. Tree-of-Thoughts (ToT)~\citep{yao2023tree} generalizes CoT through a search process. The advantage of the latter approaches is that they can be used without access to external reward functions.
\method~is a method that belongs to the {\em scaling} category that is {\em free} from external dependencies, which scales favorably compared to other parallel scaling methods such as Self-consistency.

While not training-free, ParScale~\citep{chen2025parallel} is highly relevant to our work, where the authors propose a way to perform prefix tuning for $J$ different streams, so that when decoding the next token, each stream offers a different {\em view} of the same text. We note that ParScale requires optimizing the prefix during training. In contrast, \method\ is completely training-free, as it is easy to construct different views simply by selecting different frames from the video.

\section{Method}
\label{sec:method}

\begin{figure*}[!t]
    \centering
    \includegraphics[width=\linewidth]{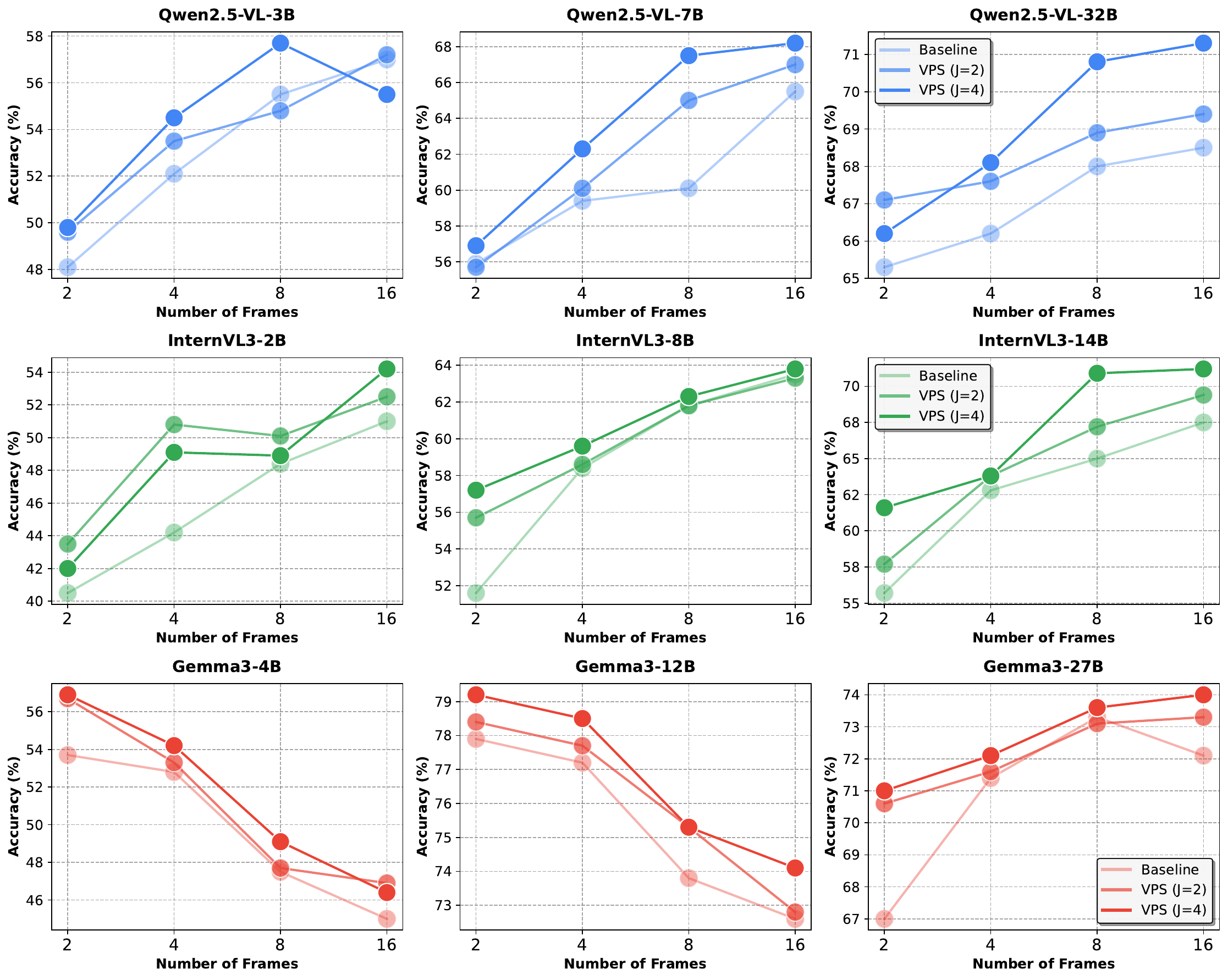}
    \caption{
    \textbf{\method\ consistently improves performance across all dimensions}. Across 3 different model classes (Qwen-2.5-VL, InternVL3, Gemma3), 3 different size (2B - 32B), and number of frames used in context, \method\ offers improved results with clearer trends with larger models. y-axis denotes the accuracy in the EventHallusion binary QA.
    }
    \vspace{-0.3cm}
    \label{fig:main}
\end{figure*}

\subsection{Video Parallel Scaling}
\label{subsec:main}

VideoLLMs take as input a subsampled version of $T$-frame input video $I_{1:T}$. Define a frame selector function $\Sc_K: \{1, \dots, T\} \mapsto \{0, 1\}$, which selects $K$ out of $T$ frames:
\begin{align}
\label{eq:frame_selector}
    \Sc_K(t) = 
    \begin{cases}
        1, & \text{if frame } t \text{ is kept}, \\[4pt]
        0, & \text{otherwise}.
    \end{cases}
\end{align}
The kept indices are $K := \{t | \Sc_K(t) = 1\} = \{t_1 < \cdots < t_{|K|}\}$. For every frame $I_{t_k}$, embeddings are extracted from the vision encoder $f^\phi$ to get $M$ soft tokens of dimension $d$
\begin{align*}
    V_{t_k} = f^\phi(I_{t_k}) = (v^1_{t_k}, \cdots, v^M_{t_k}) \in \Rd^{M \times d},
\end{align*}
which are concatenated frame-wise to get $V = [V_{t_1}, \cdots, V_{t_{|K|}}]$.
Denote $\Sigma$ as the vocabulary.
The probability of the answer $y_{1:N} \in \Sigma^N$ given the query text prompt $q_{1:L} \in \Sigma^L$ and the vision context $V$ is modeled as
\begin{align}
    p^\theta(y_{1:N}|q_{1:L}, V) = \prod_{n=1}^N p^\theta(y_n|y_{<n}, q_{1:L}, V),
\end{align}
where the probability values are obtained through the softmax of logits, i.e. 
\begin{align}
    p^\theta(y_t|y_{<t}, q_{1:L}, V) = \softmax(\z_t), \\ 
    \z_t = g^\theta(y_{<t}, q_{1:L}, V) \in \Rd^{|\Sigma|}.
\end{align}
Assume that we have a set of $J$ frame selection functions $\{\Sc_{K_1}, \cdots, \Sc_{K_J}\}$, which would lead to a different set of soft tokens $V_j$ and estimated logit-probability values. 
Then, our goal is to aggregate the predictions from each stream so that we can incorporate the information from different frames in a collaborative fashion.
\begin{align}
\label{eq:prob_avg_fusion}
    \bar{p}^\theta(y_t|y_{<t}, q_{1:L}, V) = \sum_{j=1}^J w_j p^\theta_j(y_t|y_{<t}, q_{1:L}, V_j), \\ 
    \w = [w_1, \cdots, w_J] \in \Delta^{J},
\end{align}
where $\Delta^{J}$ denotes the $J$-simplex\footnote{Alternatively, one can also aggregate the logit values. We find that both approaches lead to similar results. See App.~\ref{appendix:results} for further discussion}.
Once a token is sampled, i.e. $y_t \sim \bar{p}_\theta(y_t|\cdot)$, we concatenate the same sampled token to each stream, then iterate until the end of sequence. See Fig.~\ref{fig:concept} for an illustration.
While it is free to choose any selector functions $\Sc_j$ and weighting functions $w_j$ for VPS,
a canonical example would be using uniform sampling for all the streams but with varying offsets. For instance, for a $T = 64$ frame video with $|K| = 4$ frames and $J = 4$ streams, the following frames would be selected: $K_1 = \{0, 16, 32, 48\}, K_2 = \{4, 20, 36, 52\}, K_3 = \{8, 24, 40, 56\}, K_4 = \{12, 28, 44, 60\}$. Note that dropping $K_{2:4}$ degrades the process to baseline VideoLLM sampling. 

\paragraph{Choice of weights}
A canonical weighting would be to simply use $w_j = 1/J,\,\forall j$. In fact, under symmetry assumptions of each stream, classical results on ensembles show that setting uniform weight is optimal~\cite{claeskens2008model}. Further, equal weighting tends to be near optimal even when compared to cases where the weights are inferred from finite number of data samples~\cite{timmermann2006forecast}. 
Hence, throughout the remainder of the manuscript, we use this canonical sampling and weighting scheme if not specified otherwise. We expand on more advanced weighting schemes in Appendix~\ref{appendix:entropy_weighting}.

\subsection{Theoretical Analysis}
\label{subsec:theory}

Here, we analyze how \method~effectively scales under the lens of Chinchilla scaling law, following the presentation in~\cite{hoffmann2022training,chen2025parallel}. Through the analysis, we reveal how \method~is able to improve the performance of VideoLLMs, as well as deduce desirable strategies for frame sampling in each stream.
Throughout the section, for simplicity, we denote $x = [q_{1:L}, V]$ as the context, $y$ as the next token label, and assume $w_j = 1/J$ for simplicity. Under this notation, the true next token distribution reads $p(y|x)$. We further let $x_j = [q_{1:L}, V_j]$ denote the context induced by selecting a subset of frames shown to stream $j$ selected by $\Sc_{K_j}$. Further, denote $p_j(y|x_j)$ the prediction of stream $j$. The details of the treatment along with the proofs and derivations are deferred to Appendix~\ref{appendix:derivation}.

We first review the simplified version of Chinchilla scaling law~\citep{hoffmann2022training}\footnote{The treatment follows \cite{chen2025parallel}, ignoring the training tokens spent.} of LLMs. Concretely, the cross-entropy (CE) loss $L$ of the model with $N$ parameters follows
\begin{align}
\label{eq:simplified_chinchilla}
    L(N) = E + \frac{A}{N^\alpha},
\end{align}
where $E$ is the irreducible entropy of the natural text, and $A,\alpha$ are constants. 

\begin{table*}[!th]
\centering\small
\captionof{table}{
\textbf{Results of \method\ on different categories of each benchmark.} \method\ improves the performance consistency across categories, regardless of the model. 32 frames are used for Video-MME (Avg. length: $\sim$ 17 min.) and 8 frames are used for EventHallusion (Avg. length $\sim$ 20 sec.).
}
\resizebox{0.9\textwidth}{!}{
\begin{tabular}{l l c c c c c c c c}
\toprule
& & \multicolumn{4}{c}{\textbf{Video-MME}} & \multicolumn{4}{c}{\textbf{EventHallusion}}\\
\cmidrule(lr){3-6}
\cmidrule(lr){7-10}
Model & Method & Short & Medium & Long & Overall & Entire & Misleading & Mix & Overall \\
\midrule
\multirow{3}{*}{\makecell[l]{Qwen2.5‑VL‑7B\\\citep{bai2025qwen2}}} & Baseline & 0.656 & 0.509 & 0.441 & 0.535 & 0.579 & 0.487 & 0.843 & 0.601 \\
& \vpst{VPS ($J=2$)} & \vpst{0.662} & \vpst{\textbf{0.523}} & \vpst{0.461} & \vpst{0.549} & \vpst{0.605} & \vpst{0.560} & \vpst{0.872} & \vpst{0.650} \\
& \vpsf{VPS ($J=4$)} & \vpsf{\textbf{0.668}} & \vpsf{\textbf{0.523}} & \vpsf{\textbf{0.467}} & \vpsf{\textbf{0.553}} & \vpsf{\textbf{0.658}} & \vpsf{\textbf{0.560}} & \vpsf{\textbf{0.912}} & \vpsf{\textbf{0.675}} \\
\midrule
\multirow{3}{*}{\makecell[l]{InternVL3-8B\\\citep{zhu2025internvl3}}} & Baseline & 0.719 & 0.570 & 0.506 & 0.598 & \textbf{0.421} & 0.616 & 0.843 & 0.618 \\
& \vpst{VPS ($J=2$)} & \vpst{\textbf{0.742}} & \vpst{0.570} & \vpst{0.506} & \vpst{0.598} & \vpst{0.404} & \vpst{0.622} & \vpst{\textbf{0.852}} & \vpst{0.619} \\
& \vpsf{VPS ($J=4$)} & \vpsf{0.728} & \vpsf{\textbf{0.603}} & \vpsf{\textbf{0.533}} & \vpsf{\textbf{0.622}} & \vpsf{0.412} & \vpsf{\textbf{0.632}} & \vpsf{0.843} & \vpsf{\textbf{0.623}} \\
\midrule
\multirow{3}{*}{\makecell[l]{Gemma3-12B\\\citep{team2025gemma}}} & Baseline & 0.509 & 0.466 & 0.438 & 0.471 & 0.605 & 0.803 & 0.824 & 0.753 \\
& \vpst{VPS ($J=2$)} & \vpst{0.513} & \vpst{\textbf{0.470}} & \vpst{\textbf{0.454}} & \vpst{0.479} & \vpst{0.612} & \vpst{0.812} & \vpst{0.829} & \vpst{0.761} \\
& \vpsf{VPS ($J=4$)} & \vpsf{\textbf{0.523}} & \vpsf{0.466} & \vpsf{0.452} & \vpsf{\textbf{0.480}} & \vpsf{\textbf{0.623}} & \vpsf{\textbf{0.824}} & \vpsf{\textbf{0.833}} & \vpsf{\textbf{0.770}} \\
\bottomrule
\end{tabular}
}
\label{tab:main}
\vspace{-0.3cm}
\end{table*}

For VideoLLMs, a single stream receives a subset of frames $x_j = \Sc_{K_j}(x)$ so that every $x_j$ is incomplete in information. Concretely, let
\begin{align}
    p_j(y|x_j) = p(y|x)(1 + \Delta_j), \quad \Delta_j := \frac{p_j - p}{p},
\end{align}
where $\Delta_j$ total relative error that measures the deviation of the stream's imperfect answer $p_j$ from the true answer $p$. The relative error can be split into two parts
\begin{align}
    \Delta_j = \underbrace{\Ed_{x,y}[\Delta_j]}_{{\rm Bias}} + \underbrace{\varepsilon_j}_{{\rm Variance}}.
\end{align}
The bias is the systematic or predictable part of the error, as stream $j$ always misses a specific set of frames, and the remaining is the variance, with $\Ed_{x,y}[\varepsilon_j] = 0$ by construction. 
Further define
\begin{align}
\label{eq:bias_from_seeing_fewer_frames}
    B_j = - \Ed_{x,y}[\Delta_j], \quad \varepsilon_j = \Delta_j + B_j.
\end{align}
Then, we have the following result
\begin{restatable}[informal]{proposition}{vpsscale}
\label{prop:vps}
    A VideoLLM that is shown some subset of frames $x_j$ follow
    \begin{align*}
        L_j^{\rm VideoLLM}(N) \;\approx\; E + \frac{A}{N^\alpha} + B_j.
    \end{align*}
    Further, the expected CE loss of \method~ with $J$ streams follow
    \begin{align*}
        L^{\rm\method}(N, J) \;\approx\; E + \frac{A}{(NJ^{1/\alpha})^\alpha}[1 + (J-1)\rho] + \bar{B}(J),
    \end{align*}
    where $\bar{B}(J) = \frac{1}{J}\sum_{j=1}^J B^{(j)}$, and $\rho$ is the correlation coefficient between $\varepsilon_i, \varepsilon_j, i \neq j$.
\end{restatable}
Notice that $L^{\rm\method}$ is similar to what was shown in \cite{chen2025parallel}, but with an additional offset $\bar{B}(J)$, which stems from the fact that the VideoLLMs see partial video frames. Two conclusions can be derived.

First, we wish to keep $\rho \in [0, 1]$ as small as possible for the loss to be fast-decaying with $J$. For this, it is crucial that we select {\em distinct} frames for each stream to maximize diversity. Note that this is different from the data augmentation strategy in \cite{woo2024ritual}, which would yield high $\rho$ as both streams would see the same frames, and thus, small gains even when using parallel streams. In Sec.~\ref{sec:experiments}, we empirically show that this is indeed the case.

Second, increasing $J$ does not increase bias unless you pick a subset $x_j$ that yields high KL divergence ${\rm KL}(p \| p_j)$. To control this bias term, uniform strides of video frames per stream should be preferred, akin to how VideoLLMs are trained.
Notably, both conditions are satisfied when we use our canonical sampling scheme with uniform sampling per stream and constant phase offsets between the streams. 

\begin{figure*}[!t]
    \centering
    \includegraphics[width=\linewidth]{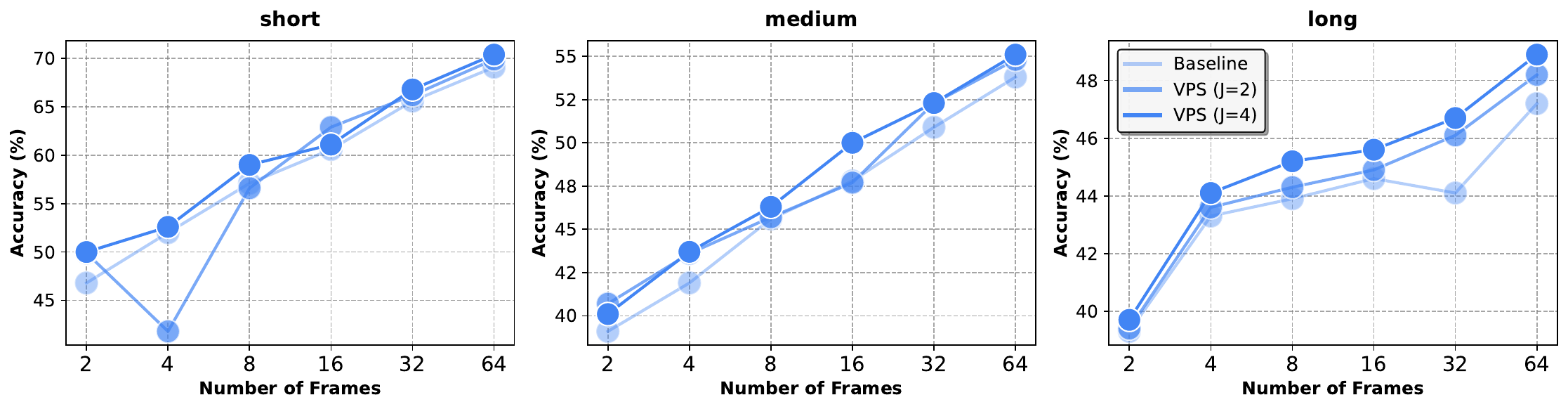}
    \caption{
    \textbf{\method\ scales better for longer videos}. Comparing the results of Qwen2.5-VL-7B on Video-MME for each category, we see a clearer trend in the long video category (15 - 30 min.).
    }
    \vspace{-0.3cm}
    \label{fig:videomme_length}
\end{figure*}

\begin{table*}[!t]
\centering
\resizebox{0.8\textwidth}{!}{
\begin{tabular}{c >{\raggedright\arraybackslash}p{1.5cm} p{11cm}}
\toprule
\multirow{3}{*}[-0.6cm]{\rotatebox{-90}{EventHallusion}} & Video & \includegraphics[width=\linewidth, valign=t]{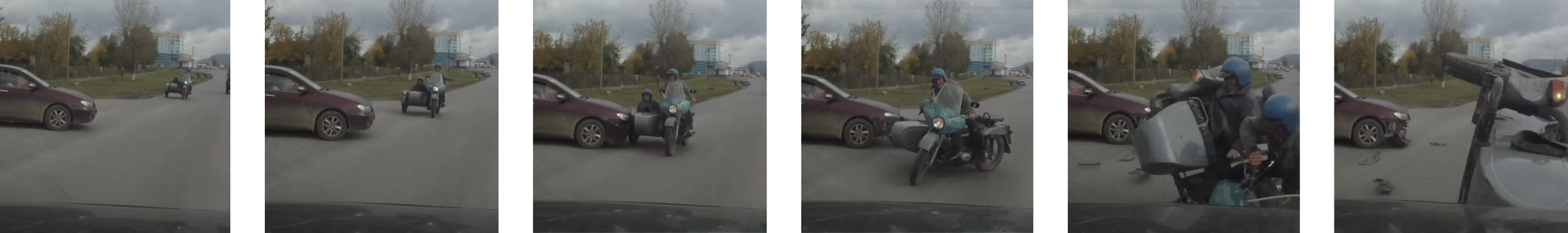}
\\
\cmidrule(l){2-3}
& Baseline & The video captures a motorcycle accident where a car and a motorcycle collide, resulting in the motorcycle being flipped over, \textcolor{GoogleRed}{while pedestrians nearby assess the situation on a rural road.} \\
\cmidrule(l){2-3}
& \method\ ($J=4$) & The video captures a collision between a car and a three-wheeled vehicle, resulting in significant damage to both vehicles, \textcolor{GoogleBlue}{as seen from the perspective of a car on the road behind them.} \\
\midrule
\multirow{2}{*}[-0.6cm]{\rotatebox{-90}{Video-MME}} & Video & \includegraphics[width=\linewidth, valign=t]{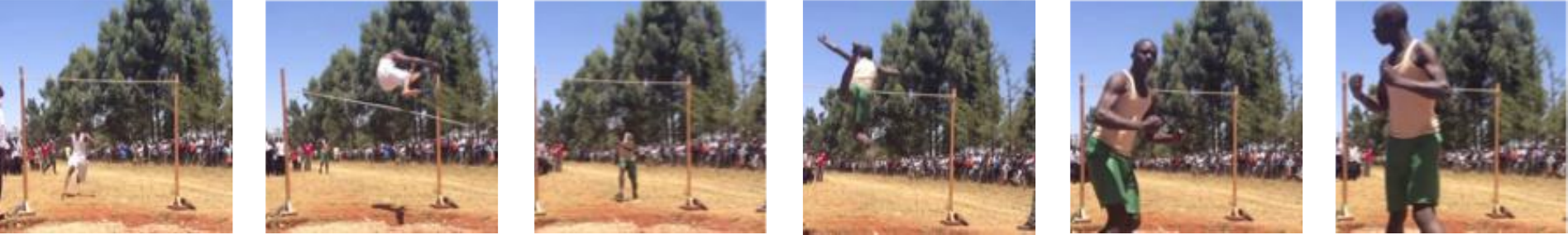}
\\
\cmidrule(l){2-3}
& Caption & Who ultimately won the high jump competition in the video?
\par
\begin{tabular}{@{} l p{9.8cm} @{}}
    Baseline: & \textcolor{GoogleRed}{A. Athlete wearing a white top and black trousers.} \\
             & B. Athlete wearing a white top and white shorts. \\
    VPS ($J=4$):     & \textcolor{GoogleBlue}{C. Athlete wearing a yellow top and green shorts.} \\
             & D. Athlete wearing a yellow top and black trousers.
\end{tabular} \\
\bottomrule
\end{tabular}
}
\caption{\textbf{Qualitative analysis of VPS.} VPS offers advantages in both free-form answering tasks, and mulitple-choice QA.}
\vspace{-0.3cm}
\label{tab:qualitative_result}
\end{table*}

\section{Experiments}
\label{sec:experiments}

\paragraph{Experimental settings}

We test our method on 3 different model classes, which are considered the state-of-the-art open-source VideoLLMs.
For Qwen2.5-VL~\citep{bai2025qwen2}, we take 3B, 7B, and 32B models. For Gemma3~\citep{team2025gemma}, 4B, 12B, and 27B models are used. For InternVL3~\citep{zhu2025internvl3}, 2B, 8B, 14B models are used. These models are evaluated across 2 different benchmarks: Video-MME~\citep{fu2025video} and EventHallusion~\citep{zhang2024eventhallusion}.
Video-MME is a general video understanding benchmark, which consists of 900 videos and 2,700 questions, categorizing the video into short, medium, and long. The length of the video ranges from a few minutes to hours. We consider the case without subtitles. EventHallusion is a benchmark focused on evaluating the hallucination of VideoLLMs, with three categories: entire, mix, and misleading. The benchmark consists of 400 videos and 711 questions, including binary QA and open-ended QA. The duration of the videos range from a few seconds to 30 seconds. Baseline refers to the standard next token sampling with subsampled video frames with a single stream.

\paragraph{\method\ consistently improves performance across all dimensions}

We test the performance of \method~by varying the number of frames used in context (2 - 16 frames for EventHallusion, 2 - 32 frames for Video-MME). In Fig.~\ref{fig:main} we see that \method~improves the performance of VideoLLMs as we increase the number of streams ($J$) used, regardless of the number of frames, the model class, and the scale. Due to the quadratic complexity of the attention operation, and the high token count of video frames, moderate-sized VideoLLMs easily face out-of-memory (OOM) issues. For instance, Qwen2.5-VL-32B model with 32 frames results in OOM on a A6000 GPU. In contrast, \method~offers a way of scaling compute to yield better performance with {\em constant} memory, achieving improvements by increasing $J$. When exceeding a certain budget, the performance of VideoLLMs tend to either plateau or decrease as one incorporates more frames into the context, as can also be observed in the plot in Fig.~\ref{fig:main}. \method~provides a viable alternative especially in this regime, by being able to scale the results with more compute in the parallel direction\footnote{This is especially evident with Gemma3, as can be seen in the bottom row of Fig.~\ref{fig:main}.}.
In Appendix Fig.~\ref{fig:main_videomme} we see a similar trend for Video-MME. Experimental details are provided in Appendix~\ref{appendix:exp_details_main}.

In Tab.~\ref{tab:main} we provide a more detailed view into how \method\ performs across different categories in each benchmark, where we see consistent improvements regardless of the category. Notice that we fixed the used frame count to 32 for Video-MME and 8 for EventHallusion, considering the average length of the videos. We further elucidate the results according to the different length and frame counts used in Fig.~\ref{fig:videomme_length}, where we show that the scaling behavior is increasingly well observed as the length of the video increases. Interestingly, the improvement from \method\ becomes more pronounced as model capacity grows and as video duration increases (Fig.~\ref{fig:main},~\ref{fig:videomme_length}), a direction where the community is already pushing.

\begin{figure*}[!t]
    \centering
    \includegraphics[width=\linewidth]{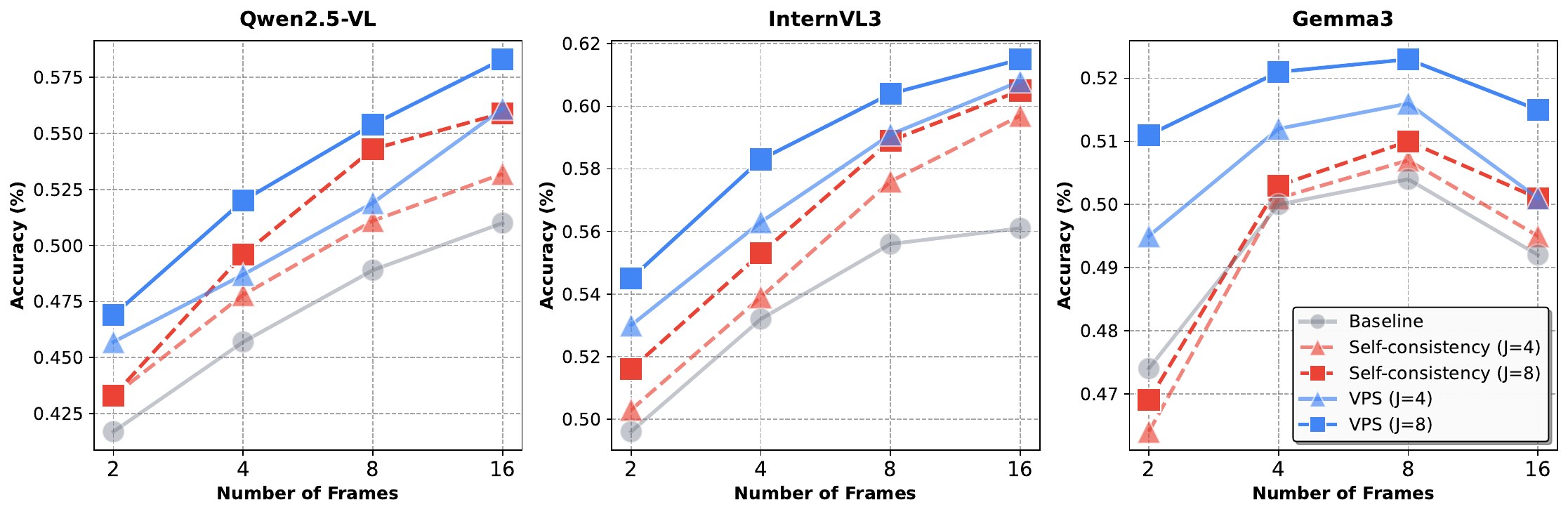}
    \caption{
    \textbf{\method\ scales favorably compared to Self-consistency.} On Video-MME, \method\ outperforms Self-consistency under the same budget by being able to incorporate information from different frames, rather than relying on the same information for all the streams.
    }
    \vspace{-0.3cm}
    \label{fig:scaling_vs_sc}
\end{figure*}

\paragraph{VPS scales favorably compared to Self-consistency}

Self-consistency~\citep{wang2023selfconsistency} is one of the few methods that allows parallel scaling without the reliance on external modules, e.g. reward functions. A natural questiona arises: Does \method~scale better than Self-consistency? In Fig.~\ref{fig:scaling_vs_sc}, we answer with a positive by showing that this is indeed the case, with the details on the experiments provided in Appendix~\ref{appendix:exp_self_consistency}. Notice that on a multiple-choice QA, where the model is forced to answer with a single token, the two strategies are equivalent with large enough sample size: 1) aggregating the probabilities from different streams, 2) sampling from each stream, and determining the answer through majority voting. In this regard, the only difference between Self-consistency and \method~in this constrained situation is whether the different streams $j$ sees different frames of video. Fig.~\ref{fig:scaling_vs_sc} emphasizes the importance of using different frames for each stream by showing superior scaling performance for \method.

\begin{table*}[t!]
\centering\small
\captionof{table}{
\textbf{\method\ outperforms baseline on free-form description task.} LLM-as-a-judge score in Likert scale, sentence similarity, and ROUGE-L scores are reported on the EventHallusion description task.
}
\resizebox{1.0\textwidth}{!}{
\begin{tabular}{llcccccccccccc}
\toprule
& \textbf{Nframe =} & \multicolumn{3}{c}{\textbf{2}} & \multicolumn{3}{c}{\textbf{4}} & \multicolumn{3}{c}{\textbf{8}} & \multicolumn{3}{c}{\textbf{16}}\\
\cmidrule(lr){3-5}
\cmidrule(lr){6-8}
\cmidrule(lr){9-11}
\cmidrule(lr){12-14}
Model & Method & LLM & STS & ROUGE & LLM & STS & ROUGE & LLM & STS & ROUGE & LLM & STS & ROUGE \\
\midrule
\multirow{2}{*}{\makecell[l]{Qwen2.5‑VL‑7B\\\citep{bai2025qwen2}}} & Baseline & 2.06 & 44.3 & 19.8 & 2.41 & 48.0 & 19.2 & \textbf{2.50} & 49.2 & 19.7 & 2.05 & 43.5 & 19.2\\
& \vpsf{VPS ($J=4$)} & \vpsf{\textbf{2.25}} & \vpsf{\textbf{46.1}} & \vpsf{19.8} & \vpsf{\textbf{2.47}} & \vpsf{\textbf{48.7}} & \vpsf{\textbf{19.5}} & \vpsf{\textbf{2.43}} & \vpsf{\textbf{49.8}} & \vpsf{\textbf{19.9}} & \vpsf{\textbf{2.52}} & \vpsf{\textbf{50.8}} & \vpsf{\textbf{19.8}}\\
\midrule
\multirow{2}{*}{\makecell[l]{InternVL3-8B\\\citep{zhu2025internvl3}}} & Baseline & 2.09 & 41.9 & 9.24 & 2.18 & 42.1 & \textbf{9.14} & 2.30 & 43.0 & \textbf{9.26} & 2.24 & 44.5 & 9.20\\
& \vpsf{VPS ($J=4$)} & \vpsf{\textbf{2.14}} & \vpsf{\textbf{43.7}} & \vpsf{\textbf{9.89}} & \vpsf{\textbf{2.23}} & \vpsf{\textbf{43.0}} & \vpsf{9.10} & \vpsf{\textbf{2.39}} & \vpsf{\textbf{45.0}} & \vpsf{9.24} & \vpsf{\textbf{2.33}} & \vpsf{\textbf{44.8}} & \vpsf{\textbf{9.29}}\\
\midrule
\multirow{2}{*}{\makecell[l]{Gemma3-12B\\\citep{team2025gemma}}} & Baseline & 1.94 & 45.8 & 18.0 & \textbf{2.26} & 48.2 & \textbf{17.8} & \textbf{2.51} & 48.7 & 17.3 & 2.33 & 48.4 & 17.9\\
& \vpsf{VPS ($J=4$)} & \vpsf{\textbf{2.05}} & \vpsf{\textbf{47.5}} & \vpsf{\textbf{18.7}} & \vpsf{2.21} & \vpsf{\textbf{48.4}} & \vpsf{17.5} & \vpsf{2.49} & \vpsf{\textbf{48.9}} & \vpsf{\textbf{18.8}} & \vpsf{\textbf{2.40}} & \vpsf{\textbf{48.5}} & \vpsf{\textbf{18.5}}\\
\bottomrule
\end{tabular}
}
\label{tab:freeform}
\end{table*}

\paragraph{Free form answering}

Another advantage of \method\ is that it can be used for free form answering such as captioning, whereas methods such as Self-consistency cannot. We test the capability of \method\ on the description task of EventHallusion~\citep{zhang2024eventhallusion}, and evaluate the performance using LLM-as-a-judge~\citep{zheng2023judging} following \cite{zhang2024eventhallusion}, sentence similarity (STS), and ROUGE-L scores. We use \code{Gemini-2.5-flash} for computing the LLM-as-a-judge score, and use the SentenceTransformer library from huggingface (Model: \code{all-MiniLM-L6-v2}) to compute STS. Experimental details are provided in Appendix~\ref{appendix:exp_free_form}. We see in Tab.\ref{tab:freeform} that VPS consistently outperforms the vanilla baseline in most metrics, with a qualitative example in Tab.~\ref{tab:qualitative_result}, illustrating how VPS successfully mitigates the hallucination caused in the baseline decoding strategy.

\paragraph{Incorporation of other decoding strategies}

There exists many different strategies that aim to enhance the decoding results zero-shot, at the expense of additional parallel computational overhead, similar to \method. 
For instance, TCD~\citep{zhang2024eventhallusion} extends contrastive decoding to videos, by constructing a negative stream where half of the frames are zeroed-out in an interleaving fashion. RITUAL~\citep{woo2024ritual} constructs an augmentation of the visual frame for collaborative decoding. Implementation details can be found in App.~\ref{appendix:exp_decoding_comp}.
Is \method\ better than these methods? More importantly, can we use \method\ in unison with these other approaches? In Tab.~\ref{tab:complementary}, we answer both of these questions with an affirmative. Here, we show that using other decoding strategies such as TCD and RITUAL further improves \method, hinting that the advantage gained through \method\ is largely orthogonal to the previous approaches.

\begin{table*}[t!]
\centering\small

\begin{minipage}[t]{0.58\textwidth}
    \centering
    \captionof{table}{
        \textbf{\method\ offers complementary improvements.} \method\ can be used together with existing inference-time decoding strategies. The same number of frames are used as in Tab.~\ref{tab:main}, with Qwen2.5-VL-7B as the base model.
    }
    \label{tab:complementary} 
    \resizebox{\linewidth}{!}{%
    \begin{tabular}{l c c c c c c c c}
    \toprule
    & \multicolumn{4}{c}{\textbf{Video-MME}} & \multicolumn{4}{c}{\textbf{EventHallusion}}\\
    \cmidrule(lr){2-5}
    \cmidrule(lr){6-9}
    Method & Short & Medium & Long & Overall & Entire & Misleading & Mix & Overall \\
    \midrule
    Baseline & 0.656 & 0.509 & 0.441 & 0.535 & 0.518 & 0.508 & 0.725 & 0.565 \\
    \vpsf{+ VPS ($J=4$)} & \vpsf{\textbf{0.668}} & \vpsf{\textbf{0.523}} & \vpsf{\textbf{0.467}} & \vpsf{\textbf{0.553}} & \vpsf{\textbf{0.605}} & \vpsf{\textbf{0.560}} & \vpsf{\textbf{0.873}} & \vpsf{\textbf{0.650}} \\
    \midrule
    TCD & 0.651 & 0.510 & 0.453 & 0.538 & 0.588 & 0.497 & 0.863 & 0.614 \\
    \vpsf{+ VPS ($J=4$)} & \vpsf{\textbf{0.668}} & \vpsf{\textbf{0.547}} & \vpsf{\textbf{0.479}} & \vpsf{\textbf{0.564}} & \vpsf{\textbf{0.605}} & \vpsf{\textbf{0.576}} & \vpsf{\textbf{0.882}} & \vpsf{\textbf{0.662}} \\
    \midrule
    RITUAL & 0.652 & 0.512 & 0.455 & 0.540 & 0.520 & 0.510 & 0.762 & 0.572 \\
    \vpsf{+ VPS ($J=4$)} & \vpsf{\textbf{0.665}} & \vpsf{\textbf{0.535}} & \vpsf{\textbf{0.472}} & \vpsf{\textbf{0.559}} & \vpsf{\textbf{0.601}} & \vpsf{\textbf{0.579}} & \vpsf{\textbf{0.885}} & \vpsf{\textbf{0.655}} \\
    \bottomrule
    \end{tabular}
    }
\end{minipage}
\hfill 
\begin{minipage}[t]{0.4\textwidth}
    \centering
    \captionof{table}{
        \textbf{Frame sampling strategy} for \method\ on EventHallusion with Qwen2.5-VL-7B depending on the number of frames. Uniform: sampling strategy used in the rest of the experiments.
    }
    \label{tab:frame_sampling} 
    \resizebox{\linewidth}{!}{%
    \begin{tabular}{l c c c c}
    \toprule
    Method & 2 & 4 & 8 & 16 \\
    \midrule
    Baseline & 0.559 & 0.594 & 0.660 & 0.655 \\
    Dense ($J=4$) & 0.546 & 0.587 & 0.665 & 0.658 \\
    BOLT ($J=1$) & 0.579 & 0.612 & 0.612 & 0.645 \\
    BOLT ($J=4$) & \textbf{0.581} & \textbf{0.630} & 0.638 & 0.643 \\
    \vpsf{Uniform ($J=4$)} & \vpsf{0.569} & \vpsf{0.623} & \vpsf{\textbf{0.675}} & \vpsf{\textbf{0.682}} \\
    \bottomrule
    \end{tabular}
    }
\end{minipage}
\vspace{-0.5cm}
\end{table*}

\paragraph{Choices on frame sampling}
The design space of \method\ includes the frame selector function $\Sc$ in \eqref{eq:frame_selector}. In Sec.~\ref{subsec:theory}, we argued that the canonical uniform sampling strategy is already a sufficiently good choice. Is this really the case empirically, and is there a better strategy? In Tab.~\ref{tab:frame_sampling}, we provide answers to these questions. See App.~\ref{appendix:exp_frame_sampling} for experimental details. First, we observe that dense sampling is largely inferior due to the excessive bias $B_j$ this yields. Second, when the number of frames used per stream is small compared to the video length, incorporating other frame sampling strategies such as BOLT~\citep{liu2025bolt} yields further improvement, as this strategy lets the VideoLLM attend to more relevant information, decreasing the bias $B_j$. Finally, when the number of frames used per stream is sufficient, the canonical sampling strategy is the winner.

\section{Conclusion}
\label{sec:conclusion}

In this work, we introduced Video Parallel Scaling (VPS), a training-free inference-time method designed to overcome the perceptual limitations of VideoLLMs. The core challenge for these models is that increasing the number of input frames for finer temporal understanding leads to prohibitive computational costs and performance degradation. VPS addresses this by processing multiple, disjoint subsets of video frames in parallel streams. By aggregating the output probabilities from these complementary views, VPS effectively expands the model's perceptual bandwidth without increasing the context length or memory footprint of any single stream.
Our theoretical analysis demonstrates that VPS contracts the Chinchilla scaling law by leveraging uncorrelated visual information from parallel streams. Extensive experiments across a diverse range of models (from 2B to 32B parameters), architectures, and benchmarks consistently validate our approach. We show that VPS provides significant performance gains, scales more favorably than alternatives like Self-consistency, and is orthogonal to other decoding strategies, allowing it to be used in concert for even greater improvements. These results establish VPS as a robust and memory-efficient framework for enhancing the temporal reasoning capabilities of modern VideoLLMs.

\paragraph{Limitations and future directions}

While VPS demonstrates consistent improvements, its current implementation presents opportunities for future refinement. Presently, VPS applies a uniform weighting to the output of each parallel stream. A promising avenue for future work would be to develop a dynamic weighting scheme. For instance, streams could be weighted based on an information-theoretic measure like the entropy of their output distributions, potentially prioritizing more ``confident'' or informative views~\citep{farquhar2024detecting}. Furthermore, the current aggregation method is a simple summation of probabilities. Exploring aggregation schemes akin to a multi-agent debate~\citep{du2023improving}, where streams can interact or influence each other's outputs before a final decision, could lead to a more nuanced and accurate final prediction.

{
    \small
    \bibliographystyle{ieeenat_fullname}
    \bibliography{main}
}

\clearpage
\onecolumn
\appendix
\section*{\centering \LARGE \textbf{Appendix}}
\section*{\centering \textbf{Supplementary Materials}}


\section{Proofs}
\label{appendix:derivation}

In \cite{chen2025parallel}, the inputs to the parallel streams are learnable transformation of the same input $x$, so that one can assume that each parallel stream follows \eqref{eq:simplified_chinchilla} in an unbiased way, leading to a simplification in the analysis of the parallel scaling law. We start by reviewing the result from \cite{chen2025parallel}.
\begin{lemma}[\cite{chen2025parallel}]
\label{lemma:parscale}
    The loss of ParScale with $J$ streams, each indicated with $p_j$, follow
    \begin{align}
    \label{eq:parscale_chinchilla}
        L^{\rm ParScale}(N, J) = E + \frac{A}{(NJ^{1/\alpha})^\alpha} D + \Oc(\Delta_j^3), \quad D := (J-1)\rho + 1,
    \end{align}
    where $\Delta_j := \frac{p_j - p}{p}$. $D$ depends on the correlation $\rho$ of the residuals between the streams. The loss decays fastest when $\rho = 0$, and degrades to the original Chinchilla law when $\rho = 1$.
\end{lemma}
\begin{proof}
    For each stream, following \eqref{eq:simplified_chinchilla}, the CE loss reads
    \begin{align}
        L_i &= \Ed_{x,y} [-\log p_j(y|x)] \\
        &= \Ed_{x,y} [-\log \{p \cdot (1 + \Delta_j)\}] \\
        &= \underbrace{\Ed_{x,y} [-\log p]}_{=: E} + \underbrace{\Ed_{x,y}[-\log (1 + \Delta_j)]}_{A/N^\alpha},
    \end{align}
    which can be deduced by comparing the loss with \eqref{eq:simplified_chinchilla}. Using the Taylor expansion $\log(1 + x) = x - \frac{x^2}{2} + \Oc(x^3)$,
    we can expand the second term of the rhs
    \begin{align}
        \Ed_{x,y}\left[
        \Delta_j - \frac{\Delta_j^2}{2} + \Oc(\Delta_j^3)
        \right] &= \Ed_{x,y}\left[\frac{p_j - p}{p}\right] + \Ed_{x,y}\left[\frac{\Delta_j^2}{2}\right] + \Oc(\Delta_j^3).
    \end{align}
    Since $\Ed_{x,y}[p_j] = p$, we have that
    \begin{align}
    \label{eq:2anlpha}
        \Ed_{x,y}\left[\Delta_j^2\right] = 2A/N^\alpha + \Oc(\Delta_j^3).
    \end{align}
    Now, we are ready to compute the loss $L$ by averaging the streams.
    Let $\Delta := \frac{1}{J}\sum_{j=1}^J \Delta_j$. The loss reads
    \begin{align}
        L^{\rm ParScale}(N, J) &= E + \Ed_{x,y}[-\log(1 + \Delta)] \\
        &= E + \Ed_{x,y}\left[\frac{\Delta^2}{2}\right] + \Oc(\Delta^3) \\
        &= E + \frac{1}{2J^2}\Ed_{x,y}\left[\sum_{j=1}^J \Delta_j^2 + 2\sum_{j<k}\Delta_j \Delta_k\right] + \Oc(\Delta_j^3) \\
        &\stackrel{\eqref{eq:2anlpha}}{=} E + \frac{A}{JN^\alpha} + \frac{1}{J^2}\Ed_{x,y}\left[\sum_{j<k}\Delta_j \Delta_k\right] + \Oc(\Delta_j^3)\\
        &= E + \frac{A}{JN^\alpha} + \frac{1}{J^2} \cdot J(J-1) \cdot \rho \frac{A}{N^\alpha} + \Oc(\Delta_j^3) \\
        &= E + \frac{A}{(NJ^{1/\alpha})^\alpha}\left[(J-1)\rho + 1\right] + \Oc(\Delta_j^3)
    \end{align}
\end{proof}

We are now ready to derive our proposition.

\addtocounter{proposition}{-1}
\begin{proposition}
\label{prop:vps_formal}
    A VideoLLM that is shown some subset of frames $x_j$ follow
    \begin{align}
    \label{eq:ce_loss_per_stream_video}
        L_j^{\rm VideoLLM}(N) = E + \frac{A}{N^\alpha} + B_j + \Oc(B_j^2) + \Oc(\Delta_j^3).
    \end{align}
    Further, the expected CE loss of \method~ with $J$ streams follow
    \begin{align}
    \label{eq:vps_loss}
        L^{\method}(N, J) = E + \frac{A}{(NJ^{1/\alpha})^\alpha}[1 + (J-1)\rho] + \bar{B}(J) + \Oc(\bar{B}(J)^2) + \Oc(\Delta_j^3),
    \end{align}
    where $\bar{B}(J) = \frac{1}{J}\sum_{j=1}^J B_j$, and $\rho$ is the correlation coefficient between $\varepsilon_i, \varepsilon_j, i \neq j$.
\end{proposition}
\begin{proof}
    First, recall the definition $B_j := -\Ed_{x,y}[\Delta_j],\, \varepsilon_j := \Delta_j + B_j$.
    The per-stream cross-entropy is
    \begin{align}
    \label{eq:vps_per_stream}
        L_j^{\rm VideoLLM}(N) &= -\Ed_{x,y}[\log p_j] = E + \Ed_{x,y}[-\log (1 + \Delta_j)].
    \end{align}
    In order for us to use Taylor approximation, first note that
    \begin{align}
    \label{eq:first_order}
        \Ed_{x,y}[\Delta_j] = \Ed_{x,y}[\varepsilon_j - B_j] = - B_j,
    \end{align}
    as $B_j$ is a constant w.r.t. the expectation in $p(x,y)$. Further, we have that
    \begin{align}
        \Ed_{x,y}[\Delta_j^2] &= \Ed_{x,y}\left[(\varepsilon_j - B_j)^2\right] \\
        &= \Ed_{x,y}[\varepsilon_j^2] - 2\Ed_{x,y}[\varepsilon_j B_j] + \Oc(B_j^2) \\
        &= \Ed_{x,y}[\varepsilon_j^2] + \Oc(B_j^2) \quad (\because \Ed[\varepsilon_j] = 0)
        \label{eq:second_order}
    \end{align}
    Substituting these back into the loss equation leads to
    \begin{align}
        L_j^{\rm VideoLLM}(N) &= E + \frac{1}{2}\Ed[\varepsilon_j^2] + B_j + \Oc(B_j^2) + \Oc(\Delta_j^3) \\
        &\stackrel{\eqref{eq:2anlpha}}{=} E + \frac{A}{N^\alpha} + B_j + \Oc(B_j^2) + \Oc(\Delta_j^3),
    \end{align}
    where we used the result from Lemma~\ref{lemma:parscale} to substitute for the variance of the residual.
    We have completed the first part of the proof.
    
    For \method, let $\bar\varepsilon := \sum_j \varepsilon_j/J$. We have
    \begin{align}
    \label{eq:vps_first_order}
        \Ed_{x,y}[\Delta] = \Ed_{x,y}[\bar\varepsilon - \bar{B}] = -\bar{B}
    \end{align}
    and
    \begin{align}
        \Ed_{x,y}[\Delta^2] &= \frac{1}{J^2}\Ed\left[\left(\sum_{j=1}^J \Delta_j^2\right)\right] \\
        &= \frac{1}{J^2}\Ed\left[\sum_{j=1}^J \Delta_j^2 + 2\sum_{j < k} \Delta_j\Delta_k\right] \\
        &= \frac{1}{J^2} \left(\frac{JA}{N^\alpha} + J\Oc(B_j^2) + J(J-1) \cdot \rho \sqrt{\Ed[\Delta_j^2]\Ed[\Delta_k^2]}\right) \\
        &= \frac{A}{(NJ^{1/\alpha})^\alpha}[1 + (J-1)\rho] + \Oc(\bar{B}(J)^2)
        \label{eq:vps_second_order}
    \end{align}
    Then, the CE loss for \method~reads
    \begin{align}
        L^{\method}(N,J) &= E + \Ed_{x,y}[-\log(1 + \Delta)] \\
        &= E - \underbrace{\Ed_{x,y}[\Delta]}_{\stackrel{\eqref{eq:vps_first_order}}{=} -\bar{B}} + \frac{1}{2}\Ed_{x,y}[\Delta^2] + \Oc(\Delta^3) \\
        &\stackrel{\eqref{eq:vps_second_order}}{=} E + \frac{A}{(NJ^{1/\alpha})^\alpha}[1 + (J-1)\rho] + \bar{B}(J) + \Oc(\bar{B}(J)^2) + \Oc(\Delta^3)
    \end{align}
\end{proof}

\begin{table}[t!]
\centering\small
\captionof{table}{
\textbf{Results of \method\ with logit averaging vs. probability averaging.} Both choices lead to similar results. Small: smallest variant (3B, 2B, 4B) / Base: base variant (7B, 8B, 12B).
}
\resizebox{1.0\textwidth}{!}{
\begin{tabular}{ccccccccccccccc}
\toprule
& & & \multicolumn{4}{c}{\textbf{Qwen2.5-VL}} & \multicolumn{4}{c}{\textbf{InternVL3}} & \multicolumn{4}{c}{\textbf{Gemma3}}\\
\cmidrule(lr){4-7}
\cmidrule(lr){8-11}
\cmidrule(lr){12-15}
& $J$ & Method & 2 & 4 & 8 & 16 & 2 & 4 & 8 & 16 & 2 & 4 & 8 & 16 \\
\midrule
\multirow{4}{*}{\textbf{Small}} & \multirow{2}{*}{2} & logit & 0.496 & \textbf{0.535} & \textbf{0.548} & 0.572 & \textbf{0.435} & \textbf{0.508} & \textbf{0.501} & \textbf{0.525} & \textbf{0.567} & \textbf{0.533} & 0.477 & 0.469 \\
& & prob & \textbf{0.498} & 0.523 & 0.540 & \textbf{0.577} & 0.433 & 0.489 & 0.491 & 0.523 & 0.550 & 0.521 & \textbf{0.491} & \textbf{0.472} \\
\cmidrule(lr){2-15}
& \multirow{2}{*}{4} & logit & 0.498 & 0.545 & \textbf{0.577} & 0.555 & 0.420 & \textbf{0.491} & 0.489 & 0.542 & \textbf{0.569} & \textbf{0.542} & 0.491 & 0.464 \\
& & prob & \textbf{0.511} & 0.545 & 0.564 & \textbf{0.596} & \textbf{0.423} & 0.474 & \textbf{0.496} & \textbf{0.545} & 0.548 & 0.526 & \textbf{0.513} & \textbf{0.486} \\
\midrule
\multirow{4}{*}{\textbf{Base}} & \multirow{2}{*}{2} & logit & 0.557 & 0.601 & \textbf{0.650} & 0.670 & 0.557 & 0.586 & \textbf{0.618} & 0.633 & \textbf{0.784} & 0.777 & 0.753 & \textbf{0.728} \\
& & prob & 0.570 & \textbf{0.633} & 0.648 & \textbf{0.689} & 0.557 & \textbf{0.587} & 0.616 & 0.633 & 0.782 & \textbf{0.779} & \textbf{0.755} & 0.725 \\
\cmidrule(lr){2-15}
& \multirow{2}{*}{4} & logit & 0.569 & 0.623 & \textbf{0.675} & \textbf{0.682} & \textbf{0.572} & \textbf{0.596} & 0.623 & 0.638 & \textbf{0.792} & \textbf{0.785} & 0.753 & \textbf{0.741} \\
& & prob & 0.569 & 0.623 & 0.660 & 0.667 & 0.565 & 0.591 & 0.623 & \textbf{0.640} & 0.790 & 0.782 & \textbf{0.759} & 0.739 \\
\bottomrule
\end{tabular}
}
\label{tab:logit_vs_prob}
\end{table}

\begin{figure}[!t]
    \centering
    \includegraphics[width=\linewidth]{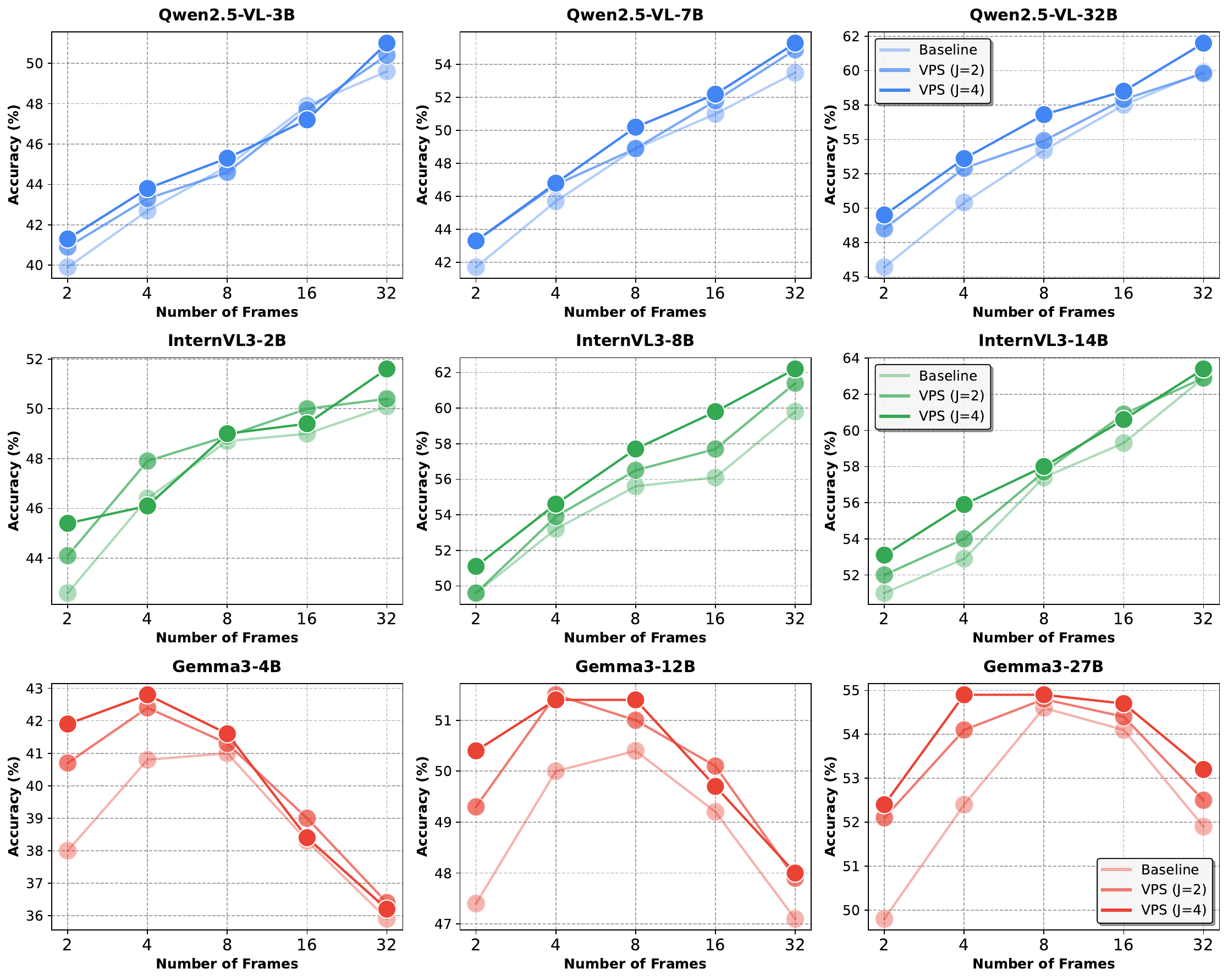}
    \caption{
    \textbf{\method\ consistently improves performance across all dimensions}. Across 3 different model classes (Qwen-2.5-VL, InternVL3, Gemma3), 3 different size (2B - 32B), and number of frames used in context. y-axis denotes the accuracy in the Video-MME multiple-choice QA.
     }
    \label{fig:main_videomme}
\end{figure}

\section{Entropy Weighting}
\label{appendix:entropy_weighting}

Our canonical weighting scheme equally weights the stream. However, in many cases, the different subsets of frames contain varying information---some may contain keyframes that directly answers the given question, while others may be less useful. A natural way to capture whether a given stream is useful is to use entropy
\begin{align}
\label{eq:def_entropy}
    H(x) = -\sum_y p(y|x)\log p(y|x).
\end{align}
When the entropy is high, it means that the output probability of the current token $y$ is close to uniform, which in turn, means that the model is less confident in its answer. In contrast, when the entropy is low, the model is more confident in its answers. Hence, while we do not have oracle access to which subsets of frames are more useful, we can leverage the output entropy as a proxy to measure this value, as also supported by recent literature~\cite{taubenfeld2025confidence,kang2025scalable,sharma2025sequential}.

\begin{wraptable}[10]{r}{0.4\textwidth}
\centering
\setlength{\tabcolsep}{0.2em}
\resizebox{0.38\textwidth}{!}{
\begin{tabular}{lcccc}
\toprule
\textbf{\# Frames} & \multicolumn{2}{c}{\textbf{8}} & \multicolumn{2}{c}{\textbf{16}} \\
\cmidrule(lr){2-3}
\cmidrule(lr){4-5}
$J$ & 4 & 8 & 4 & 8 \\
\midrule
Uniform ($\beta = 0.0$) & 0.508 & 0.528 & 0.541 & 0.552 \\
Entropic ($\beta = 7.0$) & \textbf{0.534} & \textbf{0.560} & \textbf{0.563} & \textbf{0.565} \\
\bottomrule
\end{tabular}
}
\caption{
\textbf{VPS with entropic weighting.} Results on Video-MME with Qwen2.5-VL-7B.
}
\label{tab:vps_entropic_weighting}
\end{wraptable}

We propose an entropy weighting scheme, with
\begin{align}
    w_j(x) \propto \exp\left(-\beta H_j(x)\right),
\end{align}
with $\beta \geq 0$ as a hyper-parameter. To ensure normalization, we employ softmax. Setting $\beta = 0$ reduces to our canonical uniform weighting scheme, and larger values of $\beta$ puts more emphasis on the confident streams. When the model is queried to output a direct answer, as in the benchmark settings, we can additionally use the structure of the problem. Take the multiple-choice question setting of VideoMME, where the model should answer \code{A/B/C/D}. In such cases, taking the entropy over the whole vocabulary yields noisy estimate of the model entropy. Thus, we further propose to use the {\em restricted} entropy over the vocabulary of interest\footnote{For instance, in VideoMME, we consider tokens that are upper or lowercase A/B/C/D with optional spaces.},
\begin{align}
    \bar{H}_j(x) = \sum_{c=1}^K p_j(c|x) \log p_j (c|x),
\end{align}
where $K$ is the cardinality of the vocabulary space of our interest.

\section{Further Results}
\label{appendix:results}

\paragraph{Probability and logit averaging}

In Tab.~\ref{tab:logit_vs_prob}, we compare the results of logit averaging and probability averaging when implementing VPS. Across different model classes, we find that both approaches lead to similar results. Thus, while we assume probability averaging in the theoretical analysis for simplicity, we resort to logit averaging for implementation.

\paragraph{Qualitative results}

We present additional examples of the free form description task in Tab.~\ref{tab:qualitative_result_appendix}.

\section{Experimental Details}
\label{appendix:exp_details}

\subsection{Main experiment}
\label{appendix:exp_details_main}

For Video-MME, following each question, we use the following prompt: \code{``Your response should be a single character: A, B, C, or D. Do not include any other text or explanation.''}. For EventHallusion, following each question, we use the following prompt: \code{``Please answer yes or no.''}. 
All results in the main experiments are obtained through standard sampling with temperature set to 1.0, which is different from the benchmark evaluation settings, where the standard is to resort to greedy decoding.

\subsection{Self consistency experiment}
\label{appendix:exp_self_consistency}

Notice that when decoding a single token, majority voting after decoding, and averaging the probabilities before sampling lead to the same results in the limit of infinite samples. In practice, we notice that due to formatting issues, the performance is slightly better when we use majority voting after the answer is extracted, as such scaffolding generally helps. As the goal in this experiment is to emphasize the importance of using different frames for each stream, we also implement \method\ with majority voting, but with seeing different frames for each stream, for fair comparison against Self-consistency. This way, the difference in scaling solely comes from the difference in input frames.

\begin{table}[!th]
\centering
\begin{minipage}{0.99\columnwidth}\vspace{0mm}    \centering
    \begin{tcolorbox} 
        \raggedright
        \textcolor{Blue9}{\textbf{System prompt:}\\}
        You are an intelligent chatbot designed for evaluating the correctness of generative outputs for video summaries. \\
        Your task is to compare the predicted answer with the pseudo-reference answer and determine if they match meaningfully.  \\
        You should rely more on the video frames than the pseudo-reference caption. \\
        ------ \\
        INSTRUCTIONS: \\
        - Focus on the meaningful match between the predicted answer and the pseudo-reference answer. \\
        - Consider synonyms or paraphrases as valid matches. \\
        - Evaluate the correctness of the prediction compared to the video frames. \\
        
        \textcolor{Blue9}{\textbf{User prompt:}\\}
        Given the reference caption, evaluate the quality of the predicted caption.\\
        Provide your evaluation only as an integer value between 1 and 5, with 5 indicating the highest quality. \\
        Please provide your evaluation in the following format: [evaluation] \\
    \end{tcolorbox}
    \caption{Prompts used for LLM-as-a-judge~\citep{zheng2023judging} evaluation system.}
    \label{tab:llm_as_a_judge}
\end{minipage}
\end{table}

\subsection{Free form experiment}
\label{appendix:exp_free_form}

Since the (pseudo-)ground truth answer in EventHallusion is given in a simple sentence, we use \code{``Summarize the video in one sentence.''} as the prompt. When using LLM-as-a-judge~\citep{zheng2023judging} when evaluating the free form descriptions of the video, we follow \cite{zhang2024eventhallusion} and use the prompt specified in Tab.~\ref{tab:llm_as_a_judge}. 

\subsection{Incorporating other strategies}
\label{appendix:exp_decoding_comp}

For TCD, we construct a negative stream so that the half the frames are zeroed-out in an interleaved fashion. Let $x'$ be the frame-dropped version of the sub-sampled video. Then, TCD is implemented with
\begin{align}
\label{eq:tcd}
    \tilde{p}^\theta(y|x) = (1 + \alpha) p^\theta(y|x) - \alpha p^\theta(y|x'),
\end{align}
where $\alpha \in [0, 1)$ is a constant. Additionally, we set a hyperparameter $\beta \in [0, 1]$ that keeps only the high probability tokens that exceeds the value $\beta\max_w p^\theta(y|x)$, following~\citep{li2023contrastive}. We use the default values $\alpha = 0.5, \beta = 0.1$. When used together with \method, we use $\bar{p}^\theta$ instead of $p^\theta$.
For RITUAL, we consider the following 5 augmentations: horizontal flip, vertical flip, 180 degrees random rotation, color jitter, and gaussian blur. We apply the same augmentation to each frame, and use equal weighting for the original view and the augmented view. When used together with \method, we construct an augmented view per \method\ stream.

\subsection{Frame sampling strategy}
\label{appendix:exp_frame_sampling}

The dense sampling strategy with $J$ streams first makes $J$ chunks from the video sequence. Within the chunk, the frames are uniformly sampled.
BOLT~\citep{liu2025bolt} computes the CLIP similarity between the video frames and the query prompt. A normalized score $s_i$ is then sharpened with
\begin{align}
    s_i^r = \left(
    \frac{s_i - {\rm min}(s)}{{\rm max}(s) - {\rm min}(s)}
    \right)^\alpha
\end{align},
where $\alpha = 3.0$. We then sample the frames according to this sharpened distribution. When using BOLT together with VPS, it is important that we do not sample the same frames for each stream. Hence, we eliminate the frame indices once the frame is sampled from one of the streams, then sample from the truncated distribution.
\begin{table}[!t]
\centering
\resizebox{0.8\textwidth}{!}{%
\begin{tabular}{>{\raggedright\arraybackslash}p{1.5cm} p{11cm}}
\toprule
Video & \includegraphics[width=\linewidth, valign=t]{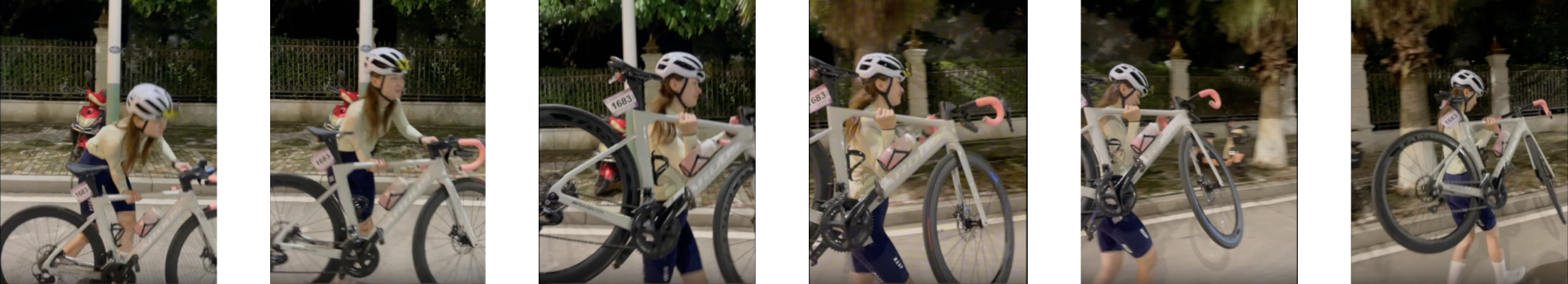}
\\
\cmidrule(l){1-2}
Baseline & A cyclist \textcolor{GoogleRed}{crashes while riding at night}, resulting in her bicycle falling over.\\
\cmidrule(l){1-2}
\method\ ($J=4$) & A cyclist is \textcolor{GoogleBlue}{struggling to carry her bike down a street.}\\
\midrule
Video & \includegraphics[width=\linewidth, valign=t]{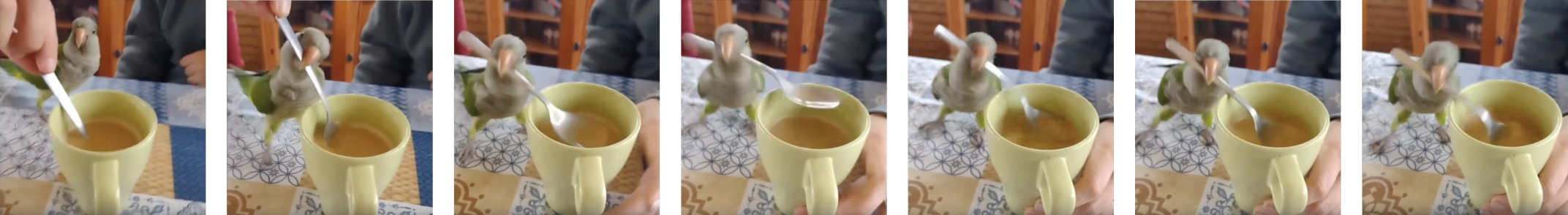}
\\
\cmidrule(l){1-2}
Baseline & A parrot curiously \textcolor{GoogleRed}{interacts with a cup of tea and a spoon}, attempting to participate in a human's tea-drinking ritual.\\
\cmidrule(l){1-2}
\method\ ($J=4$) & A parrot curiously tries to \textcolor{GoogleBlue}{stir a cup of tea with a spoon.}\\
\bottomrule
\end{tabular}
}
\caption{\textbf{Qualitative analysis of VPS.} VPS effectively captures the overall motion depicted in videos from the EventHallusion dataset.}
\vspace{-0.5cm}
\label{tab:qualitative_result_appendix}
\end{table}

\end{document}